\documentclass[conference]{IEEEtran}
\IEEEoverridecommandlockouts
\usepackage{cite}
\usepackage{amsmath,amssymb,amsfonts,amsthm}
\usepackage{algorithmic}
\usepackage{booktabs}
\usepackage{graphicx}
\usepackage{textcomp}
\usepackage{xcolor}
\def\BibTeX{{\rm B\kern-.05em{\sc i\kern-.025em b}\kern-.08em
    T\kern-.1667em\lower.7ex\hbox{E}\kern-.125emX}}
\begin{document}

\title{How Much Regularization Survives Averaging? Update Masking in Federated Learning\\

\thanks{\textsuperscript{*}Corresponding author: Zhenke Chen (chanzzkk@gmail.com).}
}

\author{\IEEEauthorblockN{1\textsuperscript{st} Wenhao Yan}
\IEEEauthorblockA{\textit{Faculty of Science \& Technology} \\
\textit{Sophia University}\\
Tokyo, Japan \\
wh\_yan@eagle.sophia.ac.jp
}
\and
\IEEEauthorblockN{2\textsuperscript{nd} Fu Kuroda}
\IEEEauthorblockA{\textit{Faculty of Science \& Technology} \\
\textit{Sophia University}\\
Tokyo, Japan \\
f-kuroda-5a8@eagle.sophia.ac.jp	
}
\and
\IEEEauthorblockN{3\textsuperscript{rd} Yucheng Jin}
\IEEEauthorblockA{\textit{Technical R\&D Department} \\
\textit{Shendian Energy Co., Ltd.}\\
Lishui, China \\
15010330370@163.com
}
\and
\IEEEauthorblockN{4\textsuperscript{th} Zhenke Chen\textsuperscript{*}}
\IEEEauthorblockA{\textit{Technical R\&D Department} \\
\textit{Shendian Energy Co., Ltd.}\\
Lishui, China \\
chanzzkk@gmail.com
}
}

\newtheorem{proposition}{Proposition}
\newtheorem{corollary}{Corollary}

\maketitle

\begin{abstract}
Federated learning on non-IID data seeks flat minima to generalize across clients, and existing methods borrow sharpness-aware minimization from centralized training. There is a second way to reach flat minima, in which the regularization comes for free from noise added to the parameter updates, and it has never been carried over to the federated setting as an implicit regularizer. We show the reason. Masking charges the optimizer for moving in sharp directions. We prove that when each client draws its own mask, federated averaging weakens that charge by exactly the cohort size, and that giving every client the same mask brings it back by a factor equal to the inverse gradient diversity of the cohort. In our experiment setting on CIFAR-10, that factor is $1.19$ out of a possible $10$. Turning off minibatch sampling raises it to $8.96$, while changing data heterogeneity a thousandfold leaves it between $1.17$ and $1.50$. The configurations keeping the regularization train far too poorly to use.
\end{abstract}


\section{Introduction}
Federated learning (FL) trains a shared model across clients that cannot pool their data \cite{mcmahan2017communication}. When those clients hold non-identically distributed data, the aggregated model is required to exhibit generalization performance across distributions it never sees jointly. Consequently, extensive research has pursued flat minima as the way to facilitate that generalization. Notable works along this direction include FedSAM \cite{qu2022generalized}, FedSMOO \cite{sun2023dynamic}, and FedGAMMA \cite{dai2023fedgamma}, importing sharpness-aware minimization from centralized training.

However, centralized training also offers a second route to flat minima, which has not been taken in FL as an implicit regularizer so far. Rather than explicitly perturbing the loss function, this kind of approaches operates directly on the update dynamics, including dropout, additive gradient noise, and most recently block-wise update masking, which discards entire parameter blocks randomly and rescales the survivors to stay unbiased, all of which induce an implicit curvature penalty through the variance of the perturbation. Masking is the sharpest instance, since discarding half of each update vector consistently outperforms most advanced dense optimizers in large language model (LLM) pre-training \cite{joo2026surprising}. Dropout and additive noise do appear in FL, as a means of reducing client compute \cite{wen2022federated} and as a privacy mechanism \cite{cheng2022differentially} respectively, but not for the curvature penalty they induce. In this paper, we find that this phenomenon is not an oversight.

Our contributions are summarized as follows.
\begin{enumerate}
    \item We show that under FedAvg with independent per-client masks, the curvature penalty induced by masking is attenuated by exactly the cohort size, and we derive the resulting penalty.
    \item Synchronizing the mask across clients restores the penalty, by a factor equal to the reciprocal of the cohort's gradient diversity. That factor reaches the cohort size at peak, and drops below one when client updates cancel, then synchronizing hurts the model performance.
    \item Measuring that factor on CIFAR-10 with FedAvg, we find it near its lower bound. Minibatch sampling noise is what actually holds it there. Data heterogeneity does not.
\end{enumerate}

\section{Preliminaries}

\subsection{Notation and Federated Averaging}

The model parameters $\theta$ are partitioned into $B$ disjoint blocks $\left\{ \theta ^{(b)} \right\}_{b=1}^B$, typically one per parameter tensor. $H_{bb}$ denotes the corresponding diagonal block of the Hessian, and $\bar{H}$ denotes the Hessian of the global objective. FedAvg proceeds in rounds. In each round, the server broadcasts $\theta$ to a cohort $C$ of $K_c$ clients randomly sampled from the whole population, each client runs $\tau$ local epochs, then the server computes the average of the resulting local updates

\begin{equation}
    \Delta_k \triangleq \theta -\theta_k,  \bar{\Delta} =\frac{1}{K_c} \sum_{k\in C}^{}\Delta_k ,
    \label{eq1}
\end{equation}

\noindent where $\theta_k$ denotes the parameters of the client $k$ after local training. We use $\tau$ instead of $E$ for the number of local epochs since $E$ denotes the expectation throughout. Note that $\tau$ counts local epochs here, not local SGD steps as in some related prior work.

\subsection{Update Masking}

Given a block-partitioned update $\Delta$ from any dense optimizer, block-wise masking replaces it with

\begin{equation}
    \tilde{\Delta}^{(b)}=\frac{1}{p} m^{(b)}\Delta^{(b)},m^{(b)}\sim Bernoulli(p),
    \label{eq2}
\end{equation}

\noindent which is drawn independently across blocks. The factor $1/p$ is the key to make the masked update unbiased, and $\mathbb{E}\left[\tilde{\Delta}^{(b)}\right] =\Delta^{(b)}$. It is that variance which produces the curvature penalty. The masking granularity fundamentally shapes the surviving curvature structure, as independent block‑wise sampling decouples coordinates and a finer partitioning consequently retains fewer cross-coordinate terms. We follow prior work in masking whole parameter tensors.

\subsection{Connection with Gradient Diversity}

For a cohort of updates $\left\{\Delta_k \right\}_{k\in C}$, we use $\Delta_S$ to denote the gradient diversity \cite{yin2018gradient} as

\begin{equation}
    \Delta_S \triangleq \frac{ {\textstyle\sum_{k}^{}\left\|\Delta_k\right\|^2}}{\left\|{\textstyle\sum_{k}^{}}\Delta_k\right\|^2}.
    \label{eq3}
\end{equation}

Throughout we work with its reciprocal $1/\Delta_S$, which measures how much individual effort of clients survives aggregation. It equals $K_c$ when the cohort moves as one, and falls below one as the updates conflict.

By the Cauchy-Schwarz inequality, we get $\Delta_S \ge 1/K_c$, with equality if and only if all $\Delta_k$ are equal. Equivalently, $1/\Delta_S \leq K_c$. Neither quantity is bounded on the other side. We have $\Delta_S>1$ and $1/\Delta_S<1$ whenever $\sum_{i<j}^{}\left\langle \Delta_i,\Delta_j\right\rangle<0$, which stands for whenever the cohort's updates point in conflicting directions and largely cancel in the average. 

FedExP explicitly connects its adaptive server step size to gradient diversity, though only in the special case of a single full-batch local step. The connection is in fact exact for any number of local steps, their step size rearranges to $(K_c/2)\cdot \Delta_S$, so the server takes larger steps precisely when the cohort disagrees \cite{jhunjhunwala2023fedexp}.

\section{Analysis}

The masking penalty was established in the centralized setting, for a single worker masking an individual optimizer step. We restate it here, then derive what becomes of it when many clients mask their round-level updates and the server averages the result.

Throughout, we let

\begin{equation}
    R(\Delta) \triangleq \frac{1-p}{2p}\sum_{b=1}^{B}{(\Delta^{(b)})}^\top H_{bb}\Delta^{(b)}
    \label{eq4}
\end{equation}

\noindent denote the curvature penalty induced by masking an update $\Delta$, where $H$ is the Hessian at the point being evaluated.

\begin{proposition}
\label{prop:mask-curvature}
Let $\Delta_t=(\Delta_t^{(1)},\dots,\Delta_t^{(B)})$ be a block-partitioned update from a base optimizer that is masked as in \eqref{eq2}, then conditioned on the state at step $t$,

\begin{equation}
    \mathbb{E}\left[l(\theta_t-\tilde{\Delta}_t)\right] = l(\theta_t-\Delta_t) + R(\Delta_t) +O(\left\|\Delta_t\right\|^3)
    \label{eq:prop1}
\end{equation}

\end{proposition}

\noindent Since the masked update is unbiased, the second term originates entirely in its variance. The quadratic form ${(\Delta_t^{(b)})}^\top H_{bb}\Delta_t^{(b)}$ measures how sharply the loss curves along the update direction within block $b$. It becomes large when the step moves into a narrow valley, small when it moves along a flat one. Only the diagonal blocks of the Hessian appear, since independent masks across blocks cancel the cross terms in expectation. Masking therefore charges the optimizer for moving in sharp directions, and the coefficient $\frac{1-p}{2p}$ sets the price, which vanishes as $p\to 1$ and grows as more of the update is discarded.

\subsection{Independent Masks}

At the beginning, the natural attempt is to let each client draw its own mask.

\begin{proposition}
\label{prop:ind}
If the masks are drawn independently across clients as well as blocks, the curvature penalty at the aggregated model, denoted $R_{ind}$, is

\begin{equation}
    R_{ind}=\frac{1}{{K_c}^2} \sum_{k\in C}R(\Delta_k).
    \label{eq:prop2}
\end{equation}

\end{proposition}

\begin{IEEEproof}
By the same expansion as in Proposition \ref{prop:mask-curvature}, and since masking leaves the mean update intact, the excess over the unmasked step is $\frac{1}{2}tr(\bar{H}Cov(\tilde{\bar{\Delta}}))$. Independence across clients gives $Cov(\tilde{\bar{\Delta}})=\frac{1}{{K_c}^2}\sum_{k}Cov(\tilde{\Delta}_k)$, and each term contributes $R(\Delta_k)$.
\end{IEEEproof}

Aggregation is linear in the updates but the penalty is quadratic, so the averaging coefficient enters squared while the sum contributes only $K_c$ terms. The result is the cohort average of the single-worker penalty, attenuated by a further factor $K_c$. Each client still incurs the full penalty along its own trajectory, only the penalty at the aggregated model is diluted.

\subsection{Synchronized Masks}

Dilution occurs because of the masks averaging out. Broadcasting a single mask seed per round, so that every client applies the same block mask, removes the necessity to average.

\begin{proposition}
\label{prop:shr}

If a single mask is drawn per round and then shared with all clients, the corresponding curvature penalty $R_{shr}$ is

\begin{equation}
    R_{shr}=R(\bar{\Delta}).
    \label{eq:prop3}
\end{equation}

\end{proposition}

\begin{IEEEproof}
    A shared mask gives $\tilde{\bar{\Delta}}^{(b)}=\frac{1}{p}m^{(b)}\bar{\Delta}^{(b)}$, therefore the aggregated inherits the covariance of a single masked update, with $\bar{\Delta}$ in place of $\Delta$.
\end{IEEEproof}

The penalty returns at full strength, and it now acts on the aggregated direction rather than on each client's local trajectory, which is the one that actually needs to be regularized. Whether it is worth anything depends on how large $\bar{\Delta}$ is relative to the individual $\Delta_k$.

\subsection{Recovery Factor}

Propositions \ref{prop:ind} and \ref{prop:shr} differ only in the object that the penalty is applied to. Taking their ratio, and treating $\bar{H}_{bb}$ as isotropic within blocks so that the quadratic forms reduce to squared norms,

\begin{equation}
    \frac{R_{shr}}{R_{ind}}=\frac{\left\|\bar{\Delta}\right\|^2}{\frac{1}{{K_c}^2} {\textstyle\sum_{k}}\left\|\Delta_k\right\|^2}=\frac{1}{\Delta _S}.
    \label{eq8}
\end{equation}

\begin{corollary}
    \label{coro:recovery}
    Synchronizing the mask recovers the curvature penalty by a factor of $1/\Delta_S$, the reciprocal gradient diversity of the cohort.
\end{corollary}

Here we find that the fraction of the masking regularization that survives federated averaging is governed by a quantity that is already measured in the aggregation of federated optimization. It is available at the server from the updates it receives. The only additional work is $K_c$ squared norms of $\Delta_k$ with no communication cost. By the previous content, synchronization recovers at most a factor $K_c$, and recovers nothing once the cohort's updates are mutually orthogonal. In the cancellation regime, the penalty is smaller than the independent one when $1/\Delta_S <1$. That means it does even worse than doing nothing. Therefore broadcasting such a shared mask is strictly counterproductive.

\subsection{Signal-Noise Decomposition}

Corollary \ref{coro:recovery} explains how much of the regularization survives, but offers no insight into the factors that determine the observed particular values. Here we divide each client's gradient update into a useful component that all clients agree on and a client-specific noise

\begin{equation}
    \Delta_k=\mu +\varepsilon_k,
    \label{eq9}
\end{equation}

\noindent where $\mu \triangleq \mathbb{E}[\Delta_k]$, so that $\varepsilon_k$ has zero mean by construction. We assume the $\varepsilon_k$ are independent across clients with common variance $\sigma^2=\mathbb{E}\left\|\varepsilon_k\right\|^2$. Then we substitute it into the two terms of $\Delta_S$ and drop the cross terms, which vanish in expectation,

\begin{equation}
    \sum_k\left\|\Delta_k\right\|^2\approx K_c(\left\|\mu\right\|^2+\sigma^2),
    \label{eq10}
\end{equation}

\begin{equation}
    \left\|\sum_{k}\Delta_k\right\|^2=\left\| K_c\mu + \sum_k\varepsilon _k\right\|^2\approx{K_c}^2\left\|\mu\right\|^2+K_c\sigma^2.
    \label{eq11}
\end{equation}

\noindent The shared component adds coherently across the cohort and grows as ${K_c}^2$ while the noise adds incoherently and grows only as $K_c$. Then we set $r=\frac{\left\|\mu\right\|^2}{\sigma^2}$,

\begin{equation}
    \frac{1}{\Delta_S}=\frac{K_c\left\|\mu\right\|^2+\sigma^2}{\left\|\mu\right\|^2+\sigma^2}=\frac{K_cr+1}{r+1},
    \label{eq12}
\end{equation}

\noindent which recovers $K_c$ as $r\to\infty$ and approaches $1$ as $r\to0$. Finally, the recovery factor is a signal-to-noise ratio in disguise, and the gap between coherent and incoherent addition is its entire content.

Meanwhile, both components are identifiable from statistics the server already forms. Let $A=\frac{1}{K_c}\textstyle{\sum_k}\left\|\Delta_k\right\|^2$ and $B=\left\|\bar{\Delta}\right\|^2$. The averaging of federated aggregation suppresses $\varepsilon_k$ but not $\mu$, so $A \approx \left\|\mu\right\|^2+\sigma^2$ while $B \approx \left\|\mu\right\|^2+\sigma^2/K_c$. $\sigma^2$ and $\left\|\mu\right\|^2$ can be separated by subtraction:

\begin{equation}
    \hat{\sigma}^2=\frac{A-B}{1-1/K_c},
    \label{eq13}
\end{equation}

\begin{equation}
    \left\|\hat{\mu}\right\|^2=A-\hat{\sigma}^2.
    \label{eq14}
\end{equation}

The reparameterization carries no information beyond $\Delta_S$ itself. But the two estimators separate a low recovery factor caused by a weak shared direction from one caused by large client noise. The estimators assume the noise is independent across clients with comparable magnitude, which is reasonable when clients hold similar amounts of data.

\section{Federated Measurement}

\subsection{Setup}

We train ResNet-18 with GroupNorm on CIFAR-10 \cite{krizhevsky2009cifar}, partitioned across $K=100$ clients by a per-class Dirichlet distribution with concentration $\alpha$. Each round the server samples a cohort of $K_c=10$ clients uniformly, each round the selected clients run $\tau$ local epochs of SGD at learning rate $0.01$ with batch size $bs$. The server applies the $n_k$-weighted mean of the updates where $n_k$ denotes the quantity of samples stored in client $k$. All runs are 200 rounds with three seeds, and we report medians over the last 50 rounds.

We vary $bs\in\left\{32,128,full\right\}$ and $\tau\in\left\{1,2,5\right\}$, where $full$ means a single gradient step over the client's entire shard, and separately $\alpha\in\left\{100,1,0.1\right\}$. GroupNorm replaces BatchNorm because batch statistics are a known confound under non-IID partitioning. Since $\bar{H}$ is not available, $\Delta_S$ is computed in the Euclidean norm rather than the $\bar{H}$-weighted one.

Correctness is established by a control in which all clients receive identical shards and identical seeds, so that every $\Delta_k$ coincides and $1/\Delta_S=K_c$ to machine precision.

We report $\Delta_S$ under the uniform convention before, and give the $n_k$-weighted value alongside it in this section. The two agree to within $0.04\%$ at $\alpha=100$, where shard sizes vary by under $3\%$. They diverge as the partition becomes unbalanced, we report the gap at $\alpha=0.1$.

\subsection{Effect of Local Training Configuration}

Two quantities distinguish the configurations in Table \ref{tab:local-steps}. The local batch size controls whether a client's gradient is exact or sampled. The number of local steps per round denoted as $T$, controls how far a client travels before reporting, which equals $\tau$ times the number of minibatches per epoch for a client holding $n_k$ samples. With about $500$ samples per client, the grid spans $T=1$ to $T=80$, from a single full-batch step at one extreme to eighty minibatch steps at the other. The two axes are not independent, since halving the batch size also doubles both the sampling noise per step and $T$, and separating their effects is the purpose of this sweep.

The pattern that emerges is not a gradient but a cliff, and it does not fall where we would expect. The three full-batch configurations give $8.96$, $4.13$, $3.96$. Every configuration using a minibatch ($32$ or $128$) lies between $1.10$ and $1.36$. The gap spans a factor of nearly three, and nothing occupies it.

A natural explanation is that clients drift apart the longer they train alone, which would make the recovery factor fall with $T$. However, it does not hold across the following two groups: five steps at full batch gives $3.96$, while four steps at a batch of $128$ gives $1.36$. What distinguishes them is not the number of steps but whether those steps use exact local gradients or sampled ones. The two also reach almost the same accuracy, $0.493$ and $0.449$, so the gap is not a matter of how far training has progressed.

It is obvious that within each group the ordering is monotone in $T$, so the number of local steps is the thing that does matter. But it simply matters far less than whether the gradients are stochastic. Raising $\tau$ from $1$ to $5$ moves the recovery factor from $1.24$ to $1.10$ at a batch of $32$, and from $8.96$ to $3.96$ at full size batch. The mechanism is most sensitive where the noise is smallest.

Corollary \ref{coro:recovery} allows $1/\Delta_S$ to fall below one, in which case a synchronized mask does worse than independent masks. It happens at a batch of $128$ with $\tau=1$ where $4.5\%$ of rounds fall in that regime. The fraction declines as the noise grows and reaches zero in the three noisiest cells. It is never large, but it shows the regime is real rather than a formal possibility. The two aggregation conventions agree to within $1\%$ in every cell, as expected at $\alpha=100$ where shard sizes vary little.

\begin{table}[t]
  \centering
  \caption{FedAvg on a near-IID split ($\alpha=100$), $K_c=10$ of 100 clients per round, 200 rounds, 3 seeds, sweeping local batch size $bs$ and epochs $\tau$ over $T=\lceil \bar{n}_k/bs\rceil\,\tau$ local steps per round ($\bar{n}_k\approx500$). Entries are medians across seeds of each run's median over the last 50 rounds; $K_c=10$ is the ceiling on $1/\Delta_S$.}
  \label{tab:local-steps}
  \footnotesize
  \setlength{\tabcolsep}{4pt}
  \begin{tabular}{@{}lrrrrrrrr@{}}
    \toprule
    & & & \multicolumn{2}{c}{$1/\Delta_S$} & & & & \\
    \cmidrule(lr){4-5}
    $bs$ & $\tau$ & $T$ & unwtd. & $n_k$-wtd. & $\hat{\mu}^2$ & $\hat{\sigma}^2$ & acc & $\substack{\text{rds.}\\<1\,\%}$ \\
    \midrule
    full & 1 & 1 & 8.96 & 8.95 & $3.6\times10^{-3}$ & $4.9\times10^{-4}$ & 0.314 & 0.0 \\
    full & 2 & 2 & 4.13 & 4.13 & $6.2\times10^{-4}$ & $1.2\times10^{-3}$ & 0.383 & 0.0 \\
    full & 5 & 5 & 3.96 & 3.96 & $3.7\times10^{-3}$ & $7.6\times10^{-3}$ & 0.493 & 0.0 \\
    \midrule
    128 & 1 & 4 & 1.36 & 1.35 & $8.4\times10^{-4}$ & $1.9\times10^{-2}$ & 0.449 & 4.5 \\
    128 & 2 & 8 & 1.32 & 1.31 & $1.2\times10^{-3}$ & $3.4\times10^{-2}$ & 0.534 & 1.7 \\
    32 & 1 & 16 & 1.24 & 1.24 & $2.1\times10^{-3}$ & $7.4\times10^{-2}$ & 0.553 & 1.0 \\
    128 & 5 & 20 & 1.18 & 1.18 & $2.8\times10^{-3}$ & $1.2\times10^{-1}$ & 0.683 & 0.0 \\
    32 & 2 & 32 & 1.19 & 1.19 & $4.1\times10^{-3}$ & $1.9\times10^{-1}$ & 0.680 & 0.0 \\
    32 & 5 & 80 & 1.10 & 1.10 & $7.3\times10^{-3}$ & $6.0\times10^{-1}$ & 0.814 & 0.0 \\
    \bottomrule
  \end{tabular}
\end{table}

\subsection{Decomposition of the Recovery Factor}

Splitting each client's update into the part the cohort agrees on and the part it does not lets us ask which of the two is responsible. Across the nine configurations the noise variance spans three orders of magnitude, from $4.9\times10^{-4}$ at full batch with $\tau=1$ to $6.0\times10^{-1}$ at a batch size of $32$ with $\tau=5$. The shared component spans far less from $6.2\times10^{-4}$ to $7.3\times10^{-3}$, and it moves in the opposite direction. Clients that take more local steps travel further along the shared direction as well as accumulating more noise. The recovery factor falls because the noise outgrows the signal by over two orders of magnitude, instead of the signal disappearing.

Configurations reach different points in training after $200$ rounds, with test accuracy ranging from $0.314$ to $0.814$, so a direct comparison risks confounding sampling noise with training progress. The grid contains a case that separates the two. The full batch run with $\tau=5$ reaches $0.493$ accuracy, between the $0.449$ from the batch size of $128$ with $\tau=1$ and the $0.534$ of the same batch with $\tau=2$. Its noise variance is lower than both by more than a factor of two, and its recovery factor is three times higher. Training progress does not place it between its neighbors. It is the sampling regime that places it apart from them.

\subsection{Trade-off with Model Accuracy}

Recovering the regularization is possible, but the configurations that recover it are not ones anybody would train with. Moving from a batch size of $32$ with $\tau=2$ to full batch with $\tau=5$ raises the recovery factor from $1.19$ to $3.96$, a gain of $3.3$ times but still reaches only $40\%$ of the ceiling $K_c$, and costs $0.187$ in test accuracy from $0.680$ to $0.493$. Pushing further to full batch with $\tau=1$, recovering buys $8.96$ at an accuracy of $0.314$.

\begin{figure}[t]
  \centering
  \includegraphics[width=\columnwidth]{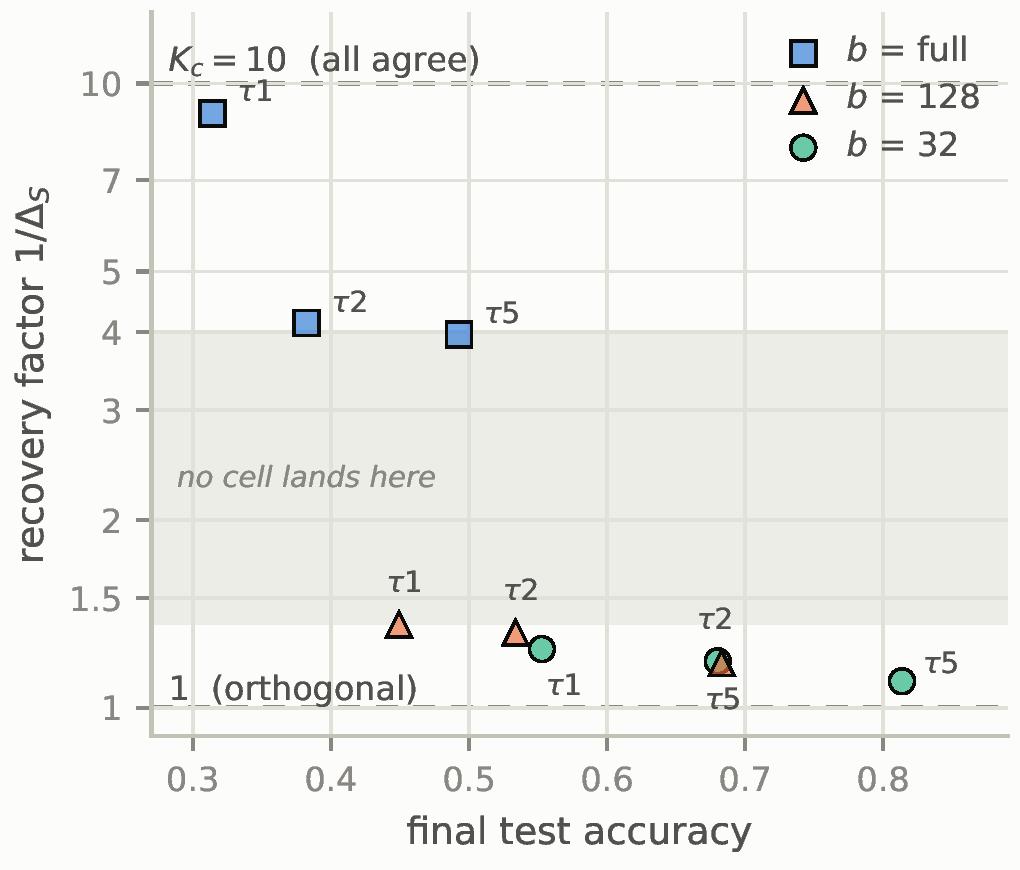}
  \caption{Recovery factor against final test accuracy for the nine configurations of Table~\ref{tab:local-steps}. Marker shape denotes local batch size, annotations denote $\tau$.}
  \label{fig:tradeoff}
\end{figure}

Fig.\ref{fig:tradeoff} plots the two quantities against each other. Only one cell is dominated, whose batch size is 128 with $\tau=1$. But the gap in the middle means there is nothing to trade along. Moving toward better recovery costs a large amount of accuracy, and even then most of the regularization is still gone.

\subsection{Effect of Data Heterogeneity}

Since the motivation is non-IID data, it is worth asking whether heterogeneity is what holds the recovery factor down. Table \ref{tab:alpha} varies the Dirichlet concentration with the local protocol fixed. As is shown in the unweighted column, $\alpha=100$ and $\alpha=1$ are indistinguishable and $\alpha=0.1$ rises to $1.50$. Under the $n_k$-weighted convention all three are flat at about $1.18$. The disagreement appears only at $\alpha=0.1$, where a few clients hold a hundred times more data than others, and it is the same in every seed, so it reflects the partition rather than run-to-run variation. The $\alpha=0.1$ runs also reach lower accuracy, so part of the difference reflects a different point in training rather than heterogeneity as such.

Neither changes the conclusion. Over three orders of magnitude in $\alpha$ the recovery factor moves between $1.17$ and $1.50$, while changing the local training protocol at a single $\alpha$ moves it from $1.19$ to $8.96$. Heterogeneity is not the main reason that holds it down.

\begin{table}[t]
  \centering
  \caption{Varying the Dirichlet concentration $\alpha$ with the local protocol fixed at $bs=32$, $\tau=2$. Shard sizes are far more uneven at $\alpha=0.1$, so the two aggregation conventions differ there.}
  \label{tab:alpha}
  \footnotesize
  \setlength{\tabcolsep}{2.5pt}
  \begin{tabular}{@{}lrrrrrr@{}}
    \toprule
    & \multicolumn{2}{c}{$1/\Delta_S$} & & & & \\
    \cmidrule(lr){2-3}
    $\alpha$ & unwtd. & $n_k$-wtd. & $\hat{\mu}^2$ & $\hat{\sigma}^2$ & acc & $\substack{\text{rds.}\\<1\,\%}$ \\
    \midrule
    100 & 1.19 & 1.19 & $4.1\times10^{-3}$ & $1.9\times10^{-1}$ & 0.680 & 0.0 \\
    1 & 1.19 & 1.17 & $3.9\times10^{-3}$ & $1.7\times10^{-1}$ & 0.626 & 0.2 \\
    0.1 & 1.50 & 1.18 & $6.5\times10^{-3}$ & $1.1\times10^{-1}$ & 0.439 & 3.0 \\
    \bottomrule
  \end{tabular}
\end{table}

\subsection{Discussion and Limitations}
Three limitations bound what these measurements support. The analysis rests on a second-order expansion that assumes small updates, which multi-step minibatch training violates. Those are precisely the configurations where the recovery factor sits near one, so the penalty there should be read as indicative rather than exact. We measure $\Delta_S$ in the Euclidean norm because the global Hessian is unavailable at the server, and Corollary \ref{coro:recovery} is exact only under the $\bar{H}$-weighted norm. And we measure how much of the penalty survives aggregation, not what it is worth. A recovery factor of $1.19$ does not by itself say masking would have helped had it survived. All results are for one architecture and one dataset.

The attenuation of Proposition \ref{prop:ind} uses only that the perturbation is zero-mean and independent across clients, so it applies equally to update-level dropout and to additive gradient noise. Perturbations that shift the mean rather than the variance, as sharpness-aware minimization does, pass through averaging undiminished, which may be why the flatness-seeking methods successfully transferred to federated learning have been of that kind.

\section{Conclusion}

We asked how much of the curvature penalty induced by update masking survives FedAvg. With independent per-client masks it is attenuated by the cohort size. Synchronizing the masks restores it by the reciprocal gradient diversity of the cohort, a factor bounded above by the cohort size but not below by one. On CIFAR-10 it sits at $1.19$ against a ceiling of $10$, and what holds it there is minibatch sampling noise rather than data heterogeneity. The configurations that recover it are not ones anybody would train with.

\bibliographystyle{IEEEtran}
\bibliography{refs}

\end{document}